\documentclass[12pt, onecolumn]{IEEEtran}

\usepackage{setspace}
\usepackage[a4paper]{geometry}
\usepackage{amsmath}
\usepackage{amsthm}
\usepackage{amssymb}
\usepackage{verbatim}
\usepackage{algorithm}
\usepackage{algorithmic}
\usepackage{graphicx}
\usepackage{epstopdf}
\usepackage{paralist}
\usepackage{color}

\ifCLASSOPTIONcompsoc
\usepackage[tight,normalsize,sf,SF]{subfigure}
\else
\usepackage[tight,footnotesize]{subfigure}
\fi

\theoremstyle{plain}

\newtheorem{lemma}{Lemma}
\newtheorem{theorem}{Theorem}
\newtheorem{definition}{Definition}
\newtheorem{corollary}{Corollary}
\theoremstyle{definition}
\newtheorem{example}{Example}
\floatname{algorithm}{Algorithm}
\newtheorem*{remark}{Remark}
\newtheorem{assumption}{Assumption}

\allowdisplaybreaks[4]

\begin{document}
\title{On Optimality of Greedy Policy for a Class of Standard Reward Function of Restless Multi-armed Bandit Problem}

\author{Quan~Liu, \qquad Kehao~Wang, \qquad Lin~Chen
\IEEEcompsocitemizethanks{\IEEEcompsocthanksitem Quan~Liu and Kehao~Wang are with the school of Information, the Wuhan University of technology, 430070 Hubei, China, and Lin~Chen is with the Laboratorie de Recherche en Informatique (LRI), Department of Computer Science, The University of Paris-Sud XI, 91405, Orsay, France. (e-mail: \{quanliu, Kehao.Wang\}@whut.edu.cn, Lin.Chen@lri.fr).}}

% "Cheat" some more by decreasing the spacing after floating figures
% and around displayed equations (this has to be after \maketitle
% which restores some defaults)
\addtolength{\floatsep}{-\baselineskip}
\addtolength{\dblfloatsep}{-\baselineskip}
\addtolength{\textfloatsep}{-\baselineskip}
\addtolength{\dbltextfloatsep}{-\baselineskip}
\addtolength{\abovedisplayskip}{-1ex}
\addtolength{\belowdisplayskip}{-1ex}
\addtolength{\abovedisplayshortskip}{-0.75ex}
\addtolength{\belowdisplayshortskip}{-0.75ex}

% PRIOR to the title within the \IEEEcompsoctitleabstractindextext IEEEtran
\IEEEcompsoctitleabstractindextext{%
\begin{abstract}
In this paper,we consider the restless bandit problem, which is one of the most well-studied generalizations of the celebrated stochastic multi-armed bandit problem in decision theory. However, it is known be PSPACE-Hard to approximate to any non-trivial factor. Thus the optimality is very difficult to obtain due to its high complexity. A natural method is to obtain the greedy policy considering its stability and simplicity. However, the greedy policy will result in the optimality loss for its intrinsic myopic behavior generally. In this paper, by analyzing one class of so-called standard reward function, we establish the closed-form condition about the discounted factor $\beta$ such that the optimality of the greedy policy is guaranteed under the discounted expected reward criterion, especially, the condition $\beta=1$ indicating the optimality of the greedy policy under the average accumulative reward criterion. Thus, the standard form of reward function can easily be used to judge the optimality of the greedy policy without any complicated calculation. Some examples in cognitive radio networks are presented to verify the effectiveness of the mathematical result in judging the optimality of the greedy policy.
\end{abstract}

}

\maketitle

% is not selected <OR> if conference mode is selected - because compsoc
% conference papers position the abstract like regular (non-compsoc)
\IEEEdisplaynotcompsoctitleabstractindextext

% \fi
\IEEEpeerreviewmaketitle

\section{Introduction}
\label{section:introduction}
We consider the system consisting of $n$ uncontrolled Markov chains evolving independently in the discrete time. Each of those chains is an independent identically-distributed (iid) two-state Markov process. The two states will be denoted as "good" state (state 1) and "bad" state (state 0). The transition probabilities is $p_{ij},i,j=0,1$. In each time instance of the system, a user is allowed to select $k$ out of the $n$ process according to its strategy, and to observe their states (assuming the precise observation), while those processes not selected by the user will evolve according to their rules. The user would obtain some reward determined by the combination of those observed states of the $k$ selected processes, i.e. collecting no reward if those states of $k$ processes are observed "bad". The above selecting, observing, and collecting process repeats until the user does not access the system. Obviously, it is a multi-armed bandit (MAB) problem~\cite{Whittle80} as well as partially observed Markov decision process (POMDP) problem which has been used and studied in the~\cite{Smallwood71}~\cite{Qzhao07}. Unfortunately, obtaining optimal solutions to a general restless bandit process is PSPACE-Hard~\cite{Papadimitriou99}, and analytical characterizations of the performance of the optimal policy are often intractable. Hence the greedy policy governing the channel selection is the suitable choice because it only focuses on maximization of the immediate reward ignoring its affect on the future reward. However, the greedy policy is not optimal generally.
\begin{comment}
\begin{figure}
\centering
\includegraphics[width=\linewidth]{Figure1.eps}
\caption{Discrete time i.i.d Markov model}
\label{fig:phasePlane}
\end{figure}
\end{comment}

Thus, recently arise two main research directions addressing the greedy policy of this kind of MAB problem. The first one is to seek the constant-factor approximation algorithm, such as 68-approximation~\cite{Guha07} developed via the linear programming relaxation under the condition of $p_{11}>0.5>p_{01}$ for each arm, and 2-approximation policy for a class of monotone restless bandit problem~\cite{Guha09}. The relevant application in dynamic multichannel access is the paper~\cite{Kliu10Index}, where the authors established the indexability and obtained Whittle index in closed form for both discounted and average reward criteria. Another research direction is to explore the optimal condition of greedy policy corresponding to a concrete application or scenario. Our work follows on this line. Although many literatures have studied this problem, the immediate reward function in those wroks only focuses on the linear combination of those observed states, i.e. in~\cite{Sahmand09}, the optimality of the greedy policy was proved in choosing $k=1$ of $N$ channels in the case of positively correlated channels, and then extended to arbitrary $k$ channels in~\cite{Sahmand09conf}. In our previous work~\cite{Wang11}, nevertheless, we have extended the work in~\cite{Sahmand09} on another line to the scenario where the immediate reward function is the simplest non-linear combination of observed states, and proved that the greedy policy is not optimal generally, which is contrary to the result of~\cite{Sahmand09conf} where the immediate reward function is the linear combination of observed states. The contrary conclusion make it necessary to study affect of the immediate reward function on the optimality of greedy policy, which is one of the major incentives for this paper.

From the technical perspective, the optimality of greedy policy needs user prefer to exploit rather than to explorer. One simplest approach to implement this mechanism is to adjust the balance between exploitation and exploration by the discounted factor $\beta$. On the other hand, noticing the different conclusion resulting from the nuance of immediate reward functions~\cite{Sahmand09conf}~\cite{Wang11}, then we only focus on one generic and basic class of immediate reward function formulated by \emph{the combination of variables of order 1}, referred to as \emph{standard} reward function.
Therefore, our objective is to derive the sufficient condition of the discounted factor such that the greedy policy is guaranteed to be optimal for the so-called \emph{standard} reward function under the discounted accumulative reward criterion. If the discounted factor $\beta=1$, the optimality of greedy policy for the discounted accumulative reward can be promoted to the optimality for the average expected reward on the time horizon of interest. Therefore, we can judge the optimality of the greedy policy for the discounted accumulative and average expected reward according to the closed-form condition of $\beta$. To the best of our knowledge, very few results been reported from this perspective.

Compared with other existing works on the optimality of greedy policy in MAB problem, and our contribution is three-fold:
\begin{itemize}
\item We analyze one special class of MBA problem where the immediate reward function is so-called standard one, and derive that the discounted accumulative reward function also is standard reward function. Furthermore, we establish the optimality of greedy policy under the discounted accumulative reward  criterion when $p_{11}>p_{01}$. The theoretical results demonstrate that the greedy policy choosing the best 1 or $N-1$ out of $N$ channels is optimal when $0 < \beta \leq 1$. For the case of choosing $k$ $(1 < k < N-1)$ channels, the greedy policy is optimal only when the discounted factor satisfies a simple closed-form condition.

 \item The major technique developed in this paper is largely based on the analytic properties of standard reward function, completely different from~\cite{Sahmand09}~\cite{Sahmand09conf} relying on the coupling argument. Besides significant and practical application in cognitive radio networks, this technique serves as the key criterion to judge the optimality of greedy policy when the immediate reward function is the combination of the standard functions in other scenarios.

\item We analyze two practical models in the cognitive radio networks. The first model in cognitive radio networks involves the sensing order problem where the secondary user selects $k$ $(1<k<N)$ of $N$ channels in order to maximize the probability of finding an idle channel. It is obvious that the immediate reward function is the order 1 non-linear combination of the availability probabilities of selected channels. The result demonstrates that the greedy policy is not optimal generally under the average expected reward, which is coherent with~\cite{Wang11}. The second model is that a user chooses $k (1\leq k <N)$ channels to access and receive a reward on the channel in good state. The immediate reward function is the linear combination of the availability of those selected channels. Our derived result is consistent with that in~\cite{Sahmand09}~\cite{Sahmand09conf} where the myopic policy choosing any number of channels is optimal.
\end{itemize}
The rest of the paper is organized as follows: Our model is formulated in Section~\ref{section:Problem_formulation}. Section~\ref{section:stand_reward} analyzes standard reward function. Section~\ref{section:optimality} gives the optimality theorem of the myopic policy. Three applications are given in Section~\ref{section:application}. Finally, our conclusions are summarized in Section~\ref{section:conclusion}.

\section{Problem Formulation}
\label{section:Problem_formulation}
As outlined in the introduction, we consider a user trying to access the system consisting of $n$ independent and statistically identical channels, each given by a two state Markov chain. The set of $n$ channels is denoted by $\mathcal{N}$, each indexed by $i = 1,2,...,n$, and the state of channel $i$ denoted by $S_i(t)=\{1~\text{(good)}, 0~\text{(bad)}\}$. The system operates in discrete time steps indexed by $t$ ($t = 1,2,...,T$), where $T$ is the time horizon of interest (or the user gives up accessing the system). Specifically, we assume that channels go through state transition at the beginning of slot $t$ and then at time $t$ the user makes the channel selection decision. Limited by hardware or sensing policy, at time $t$ the user is allowed to choose $k$ ($1\leq{k}<n$) of the $n$ channels to sense, the chosen channel set denoted by ${a^k(t)\subset{\mathcal{N}}}, |a^k(t)|=k$.

%If at least one of these $k$ channels is sensed to be good (idle state), the user transmits and collects the $1-k{\delta}$ unit of reward, where the constant,$\delta$, is ratio of sensing time of one channel to the slot time. If none is sensed good, the user does not transmit, collects no reward, and waits until slot $t+1$ to make another choice decision. This process repeats sequentially until the time horizon expires.

Obviously, the user cannot observe the whole states $\mathbf{S}(t)=[0,1]^{n}$ of the underlying system (i.e., the states of $n$ channels). We know that a sufficient statistic of such a system for optimal decision making, or the information state of the system, is given by the conditional probabilities of the state each channel is in given all past actions and observations~\cite{Smallwood71}. We denote this information state (also called belief vector) by $\Omega(t) = [\omega_{1}(t),..., \omega_{n}(t)]\in[0,1]^n$, where $\omega_{i}(t)$ is the conditional probability that channel $i$ is in state 1 at time $t$ given all past states, actions and observations. In the rest of the paper, $\omega_{i}(t)$ will be referred to as the information state of channel $i$ at time $t$, or simply the channel probability of $i$ at time $t$. Due to the Markovian nature of the channel model, the future information state is only a function of the current information state and the current action, i.e., it is independent of past history given the current information state and action. Given that the information state at time $t$ is $\Omega(t) \triangleq \{\omega_i(t), i\in{\cal N}\}$ and the sensing policy $a^k(t)\subset \mathcal{N}$ is taken, the state at time $t+1$ can be updated using Bayes Rule as shown in~\eqref{eq:belief_update}.
\begin{equation}
\omega_i(t+1)=
\begin{cases}
p_{11}, & i\in a^k(t), S_i(t)=1 \\
p_{01}, & i\in a^k(t), S_i(t)=0 \\
\tau(\omega_i(t)), & i\not\in a^k(t)
\end{cases}.
\label{eq:belief_update}
\end{equation}
where, $\tau(\omega_i(t))=\omega_i(t)p_{11}+[1-\omega_i(t)]p_{01}$.

The objective is to maximize the discounted accumulative reward over a finite horizon given in the following problem:
\begin{equation}
\label{eq:problem1}
    \max_{\pi} \emph{E}^{\pi}[\sum^{T}_{t=1} \beta^t R_{\pi_{t}}(\Omega(t))|\Omega(1)]
\end{equation}
where $R_{\pi_{t}}(\Omega(t))$ is the reward collected under state $\Omega(t)$ when channels in the set $a^k(t) = \pi_t(\Omega(t))$ are selected, $\pi_t$ specifies a mapping from the current information state $\Omega(t)$ to a channel selection action $a^k(t) =\pi_t(\Omega(t))\subset{\mathcal{N}}$.

Let $V_{t}(\Omega)$ be the value function, which represent the maximum expected discounted accumulative reward obtained from $t$ to $T$ given the initial belief vector $\Omega$. Let $p_{01}[x]$ and $p_{11}[x]$ denote the vector $[p_{01},\cdots,p_{01}]$ and $[p_{11},\cdots,p_{11}]$ of length $x$. Thus, we arrive at the following optimality equation:
\begin{eqnarray}
V_{T}(\Omega(t))=&&
  \max_{\substack{a^k(t)\subset \mathcal{N} }} \emph{E}[R(\Omega(t))]= \max_{\substack{a^k(t) \subset \mathcal{N} }} F(\Omega(t)) \\
V_{t}(\Omega(t))=&&
\label{eq:objective}
  \max_{\substack{ a^k(t) \subset \mathcal{N} }}[ F(\Omega(t))+ \beta K_t(\Omega(t)) ] \\
K_t(\Omega(t))=&&
  \sum_{e\in{\mathcal{P}(a^k(t))}} \prod_{i\in{e}}\omega_i \prod_{j\in{a^k(t)\backslash{e}}}(1-\omega_j)V_{t+1}(p_{11}[|e|],\tau(\omega_{{k+1}}(t)),\cdots,\tau(\omega_{{n}}(t)),p_{01}[k-|e|])
\end{eqnarray}
where, $\mathcal{P}(a^k(t))$ represents the power set generated by the set $a^k(t)$, the expected immediate reward $F(\Omega(t))$ is $F:\Omega(t)\rightarrow R$, and $e$ is the cardinality of set $e$. On right side of the above formulation~\eqref{eq:objective}, the reward that can be collected from slot $t$ consists of two parts: the expected immediate reward $F(\Omega(t))$ and the future discounted accumulative reward $\beta K_t(\Omega(t))$ calculated by summing over all possible realizations of the $k$ selected channels. In $K_t(\Omega(t))$, the channel state probability vector consists of three parts: a sequence of $p_{11}$'s indicating those channels sensed to be in state $1$ at time $t$; a sequence of values $\tau{(\omega_j)}$ for all $j\notin{a^k}$; and a sequence of $p_{01}$'s indicating those channels sensed to be in state $0$ at time $t$.

Considering the computational complexity of the recursive structure~\eqref{eq:objective}, we should seek other policies but not optimal policy. One of the simplest approach is a greedy policy where at each time step the objective is to maximize the expected immediate reward $F(\Omega(t))$. Thus, the greedy policy is given as follows:
\begin{equation}
\label{eq:greedy_policy}
   \widehat{a}^k(t)=arg \max_{a^k(t)\subset{\mathcal{N}}} F(\Omega(t))
\end{equation}
Note we always assume that the greedy policy, $ \widehat{a}^k(t)$, is the optimal policy at slot $t$ in the rest of paper, and then derive the sufficient condition of $\beta$ to guarantee the optimality of the greedy policy. Without introducing ambiguity, $ \widehat{a}^k(t)$ and $a^k(t)$ would be used alternatively in the rest.

\section{Standard Reward Function }
\label{section:stand_reward}
\subsection{Feature of Immediate Reward Function}
For simplicity, we assume that $\omega_{1}(t)\geq{\omega_{2}(t)}\geq\cdots\geq{\omega_{k}(t)}$, and then use $a^k(t)=\{1, \cdots k\}$ and $a^k(t)=\{\omega_{1}(t), ..., \omega_{k}(t)\}$  alternatively. The immediate reward $F({\Omega}(t))=F(\omega_{1}(t),...,\omega_{k}(t),...,\omega_{n}(t))=F(\omega_{1}(t),...,\omega_{k}(t))$ means choosing the first $k$ channels. Especially, we drop the time slot index of $\omega_i(t)$, and abuse $\omega_i(t)$ and $\omega_i$ alternatively without introducing ambiguity.

 Three fundamental while natural assumptions about the immediate reward functions are listed as follows:

\begin{assumption}
\label{assumption:symmetry}
    (symmetry) The immediate reward function $F({\Omega}(t))$ is symmetric about any two different channels in $a^k(t)$, that is, $i,j \in a^k(t)$, such that
    \begin{equation}
        F(\omega_{1}(t),...\omega_{i}(t),...,\omega_{j}(t),...\omega_{n}(t)) = F(\omega_{1}(t),...\omega_{j}(t),...,\omega_{i}(t),...\omega_{n}(t)), 1\leq{i\neq{j}}\leq{k}
    \end{equation}
\end{assumption}

\begin{assumption}
\label{assumption:affine}
    (affine) The immediate reward function $F(\Omega(t))$ is order 1 \footnote{$F({\Omega}(t))$ is affine in each variable if all other variables hold constant} polynomial of $\omega_{i}(t), 1\leq{i}\leq{n}$, that is,
    \begin{multline}
        F(\omega_{1}(t),...,\omega_{i-1}(t),\omega_{i}(t),\omega_{i+1}(t),...,\omega_{n}(t)) = \\
        \omega_{i}(t) F(\omega_{1}(t),...,\omega_{i-1}(t),1,\omega_{i+1}(t),...,\omega_{n}(t)) \\
        + (1-\omega_{i}(t))F(\omega_{1}(t),...,\omega_{i-1}(t),0,\omega_{i+1}(t),...,\omega_{n}(t))
    \end{multline}
\end{assumption}

\begin{assumption}
\label{assumption:monotonicity}
    (monotonicity) The immediate reward function $F(\Omega(t))$ increases monotonically with $\omega_{i}(t), 1\leq{i}\leq{k}$, that is,
    \begin{equation}
        \omega'_{i}(t)\geq {\omega_{i}(t)} \Rightarrow F(\omega_{1}(t),...,\omega'_{i}(t),...\omega_{n}(t)) \geq F(\omega_{1}(t),...,\omega_{i}(t),...\omega_{n}(t))
    \end{equation}
\end{assumption}
Note these assumptions are necessary and non-redundant. Moreover, these three assumptions are used to define a class of general functions, referred to as \emph{standard} immediate reward functions.

\begin{definition}
A reward function is standard one if it satisfies the aforementioned three assumptions.
\end{definition}

\begin{comment}
\\subsection{Accumulation}
 Preposition 4. The immediate reward $F(\omega_{1}(t),...,\omega_{i}(t),...,\omega_{n}(t))$ increases with $k$, which means that sensing more channels will collect more reward.
\begin{equation*}
    F(\omega_{1}(t),...,\omega_{n}(t)) \leq F(\omega_{1}(t),...,\omega_{n}(t),...,\omega_{n+m}(t)), m\geq{0}
\end{equation*}
\end{comment}

In order to see the intrinsic structure of the standard immediate reward function, we give three basic examples.

\begin{example}
Considering the scenario in~\cite{Sahmand09conf} where the user gets one unit of reward for each channel sensed good. In this example, the expected slot reward function is $F(\Omega)=\sum_{i=1}^k\omega_i$. It can be easily verified that $F$ satisfies the above three assumptions and thus is \emph{standard}.
\end{example}

\begin{example}
Considering the scenario where the user gets one unit of reward only if all the channels are sensed to be good. Thus the immediate reward is formulated by $F(\Omega)=\prod_{i=1}^k \omega_i$, which is \emph{standard} one.
\end{example}

\begin{example}
Consider the scenario in~\cite{Wang11} where the user gets one unit of reward if at least one channel is sensed good. In this case, the expected slot reward function is $F(\Omega)=1-\prod_{i=1}^k(1-\omega_i)$, which is \emph{standard} by satisfying the three assumptions.
\end{example}

\subsection{Feature of Accumulative Reward Function}
In this part, some important features of the accumulative reward function $V_t(\Omega(t))$ (also called value function) will be proved, which consists of the proof base of the optimality of greedy policy in the next section.
\begin{lemma}
\label{lemma:symmetry}
    (symmetry) $V_t(\Omega(t))$ is symmetric about $\omega_i(t)$, $\omega_j(t)$, $1 \leq i, j \leq k$, that is,
    \begin{equation*}
        V_t(\omega_{1}(t),...\omega_{i}(t),...,\omega_{j}(t),...\omega_{n}(t)) = V_t(\omega_{1}(t),...\omega_{j}(t),...,\omega_{i}(t),...\omega_{n}(t)), 1\leq{i\neq{j}}\leq{k}
    \end{equation*}
\end{lemma}
\begin{proof}
 (1)According to assumption~\ref{assumption:symmetry}, for any $1\leq{i\neq{j}}\leq{k}$ in time slot $T$,
since, $V_T(\Omega(T))=F(\Omega(T))$, then it is easy to verify $V_T(\Omega(T))$ is symmetric.

(2)Assume $V_{T-1}(\Omega(t))$, ..., $V_{t+2}(\Omega(t))$, $V_{t+1}(\Omega(t))$ are true, then at time $t$
\begin{equation*}
\begin{split}
  V_{t}(\Omega(t)) &= F(\Omega(t))+ \beta K_t(\Omega(t))
\end{split}
\end{equation*}
Based on assumption~\ref{assumption:symmetry}, $F(\Omega(t))$ is symmetric. By Lemma~\ref{lemma:future_reward_symmetry} (Appendix~\ref{appendix:future_reward_symmetry}), the second term, $K_t(\Omega(t))$ of the above formulation is symmetric. Hence, $V_{t}(\Omega(t))$ is symmetric.
\end{proof}

\begin{lemma}
\label{lemma:affine}
    (affine) $V_t(\Omega(t))$ is an affine function of $\omega_{i}(t)$, $1\leq{i}\leq{n}$ when all other $\omega_{j}(t)$, $j \neq i$, $1\leq{j}\leq{n}$ hold constant.
\end{lemma}
\begin{proof}
(1) According to assumption~\ref{assumption:affine}, in time slot $T$, $F(\Omega(T))$ is affine function of $\omega_{i}(T)$, $1\leq{i}\leq{n}$. Hence, $V_{T}(\Omega(T))=F(\Omega(T))$ is also affine function of $\omega_{i}(T)$.

(2) Assume $V_{T-1}(\Omega(T-1))$,...,$V_{t+2}(\Omega(t+2))$, $V_{t+1}(\Omega(t+1))$ are affine functions, we prove it also holds for slot $t$. Two cases should be considered as follows:

Case 1: channel $\omega_i\notin{a^k(t)}=\{\omega_1,...,\omega_k\}$:
\begin{equation*}
\begin{split}
   V_{t}(\Omega(t)) &= F(\Omega(t))+ \beta \sum_{e\in{\mathcal{P}(a^k(t))}} \prod_{p\in{e}}\omega_p \prod_{q\in{a^k(t)\backslash{e}}}(1-\omega_q)V_{t+1}(p_{11}[|e|],\tau({\omega_{k+1}}),...,\tau({\omega_{n}}),p_{01}[k-|e|])
   \end{split}
\end{equation*}
Since $F(\Omega(t))$ is unrelated with $\omega_i$, $V_{t+1}(\Omega(t+1))$ is the affine function of $\omega_i$ by the induction hypothesis and $\tau(\omega_i)$ is an affine transform of $\omega_i$, we have $V_{t}(\Omega(t))$ is the affine function of $\omega_i$.

Case 2: channel $\omega_i\in{a^k(t)}$, let $a^{k-1}(t)=a^k(t)-\{\omega_i\}$, we have
\begin{equation*}
\begin{split}
   V_{t}(\Omega(t)) &= F(\Omega(t))+\beta \sum_{e\in{\mathcal{P}(a^k(t))}} \prod_{p\in{e}}\omega_p \prod_{q\in{a^k(t)\backslash{e}}}(1-\omega_q)V_{t+1}(p_{11}[|e|],\tau({\omega_{k+1}}),...,\tau({\omega_{n}}),p_{01}[k-|e|])\\
   & = F(\omega_1,...,\omega_i,...,\omega_k)+\beta \sum_{m=0}^{k-1}\sum_{\substack{|e|=m \\ e\in{\mathcal{P}(a^{k-1}(t))}}} \prod_{p\in{e}}\omega_p \prod_{q\in{a^{k-1}(t)\backslash{e}}}(1-\omega_q) \{ \\
   & ~~~~~\omega_i V_{t+1}(p_{11}[|e|],p_{11},\tau({\omega_{k+1}}),...,\tau({\omega_{n}}),p_{01}[k-|e|]) \\
   & ~~~~~~~+(1-\omega_i) V_{t+1}(p_{11}[|e|],\tau({\omega_{k+1}}),...,\tau({\omega_{n}}),p_{01},p_{01}[k-|e|])
   \end{split}
\end{equation*}
By assumption 2, $F(\omega_1,...,\omega_i,...,\omega_k)$ is the affine function of $\omega_i$. The second term of the right hand of the above formulation is also the affine function of $\omega_i$. Therefore, $V_{t}(\Omega(t))$ is the affine function of $\omega_i$.
Combining the two cases, we have $V_{t}(\Omega(t))$ is the affine function of $\omega_i$.  Lemma~\ref{lemma:affine} is concluded.
\end{proof}

\begin{lemma}
\label{lemma:monotonicity}
    (monotonicity) $V_t(\Omega(t))$ increases monotonically with $\omega_i, 1\leq{i}\leq{n}$, that is,
    \begin{equation*}
        \omega^{'}_{i}(t)\geq {\omega_{i}(t)} \Rightarrow V_t(\omega_{1}(t),...,\omega^{'}_{i}(t),...\omega_{n}(t)) \geq V_t(\omega_{1}(t),...,\omega_{i}(t),...\omega_{n}(t)),  1\leq{i}\leq{n}
    \end{equation*}
\end{lemma}
\begin{proof}
(1) The lemma holds trivially for slot $T$ considering $V_T(\Omega(T))=F(\Omega(T))$, which is the increasing function with $\omega_i$.

(2)Assume $V_{T-1}(\Omega(T-1))$,...,$V_{t+2}(\Omega(t+2))$, $V_{t+1}(\Omega(t+1))$ increase monotonically, we prove it is true for slot $t$ by two different cases.

Case 1: channel $\omega_i\notin{a^k(t)}$:
\begin{equation*}
\begin{split}
   V_{t}(\Omega(t)) &= F(\Omega(t))+\beta \sum_{e\in{\mathcal{P}(a^k(t))}} \prod_{p\in{e}}\omega_p \prod_{q\in{a^k(t)\backslash{e}}}(1-\omega_q)V_{t+1}(p_{11}[|e|],\tau({\omega_{k+1}}),...,\tau({\omega_{n}}),p_{01}[k-|e|])
   \end{split}
\end{equation*}
Since $F(\Omega(t))$ is unrelated with $a^k(t)$, $V_{t+1}(\Omega(t+1))$ increases with $\omega_i$ by the induction hypothesis and $\tau(\omega_i)$ increases with $\omega_i$ when $p_{11}>p_{01}$, we have $V_{t}(\Omega(t))$ is the increasing function of $\omega_i$.

Case 2: channel $\omega_i\in{a^k(t)}$, let $a^{k-1}(t)=a^k(t)-\{\omega_i\}$, we have
\begin{equation*}
\begin{split}
   V_{t}(\Omega(t)) &= F(\Omega(t))+\beta \sum_{e\in{\mathcal{P}(a^k(t))}} \prod_{p\in{e}}\omega_p \prod_{q\in{a^k(t)\backslash{e}}}(1-\omega_q)V_{t+1}(p_{11}[|e|],\tau({\omega_{k+1}}),...,\tau({\omega_{n}}),p_{01}[k-|e|])\\
   & = F(\omega_1,...,\omega_i,...,\omega_k)+\beta \sum_{m=0}^{k-1}\sum_{\substack{|e|=m \\ e\in{\mathcal{P}(a^{k-1}(t))}}} \prod_{p\in{e}}\omega_p \prod_{q\in{a^{k-1}(t)\backslash{e}}}(1-\omega_q) [ \\
   & ~~~~~\omega_i V_{t+1}(p_{11}[|e|],p_{11},\tau({\omega_{k+1}}),...,\tau({\omega_{n}}),p_{01}[k-|e|]) \\
   & ~~~~+(1-\omega_i) V_{t+1}(p_{11}[|e|],\tau({\omega_{k+1}}),...,\tau({\omega_{n}}),p_{01},p_{01}[k-|e|]) ]\\
   & = F(\omega_1,...,\omega_i,...,\omega_k)+ \sum_{m=0}^{k-1}\sum_{\substack{|e|=m \\ e\in{\mathcal{P}(a^{k-1}(t))}}} \prod_{p\in{e}}\omega_p \prod_{q\in{a^{k-1}(t)\backslash{e}}}(1-\omega_q) [ \\
   &~~~~~  \omega_i [V_{t+1}(p_{11}[|e|],p_{11},\tau({\omega_{k+1}}),...,\tau({\omega_{n}}),p_{01}[k-|e|])\\
   &~~~~~~~~~~~~~~~~ -V_{t+1}(p_{11}[|e|],\tau({\omega_{k+1}}),...,\tau({\omega_{n}}),p_{01},p_{01}[k-|e|]) ] \\
   & ~~~~+V_{t+1}(p_{11}[|e|],\tau({\omega_{k+1}}),...,\tau({\omega_{n}}),p_{01},p_{01}[k-|e|]) ]
   \end{split}
\end{equation*}
The first term, $F(\omega_1,...,\omega_i,...,\omega_k))$, of the right hand of the above formulation increases monotonically with $\omega_i$, and the second term also is the increasing function of $\omega_i$ because
\begin{equation}
\label{eq:monotonicity_case2}
\begin{split}
&V_{t+1}(p_{11}[|e|],p_{11},\tau({\omega_{k+1}}),\tau({\omega_{k+2}}),\cdots,\tau({\omega_{n-1}}),\tau({\omega_{n}}),p_{01}[k-|e|])\\
& ~~~~   -V_{t+1}(p_{11}[|e|],\tau({\omega_{k+1}}),\tau({\omega_{k+2}}),\cdots,\tau({\omega_{n-1}}),\tau({\omega_{n}}),p_{01},p_{01}[k-|e|]) \\
&= [ V_{t+1}(p_{11}[|e|],p_{11},\tau({\omega_{k+1}}),\tau({\omega_{k+2}}),\cdots,\tau({\omega_{n-1}}),\tau({\omega_{n}}),p_{01}[k-|e|])\\
& ~~~~   -V_{t+1}(p_{11}[|e|],\tau({\omega_{k+1}}),\tau({\omega_{k+1}}),\tau({\omega_{k+2}}),\cdots,\tau({\omega_{n-1}}),\tau({\omega_{n}}),p_{01}[k-|e|]) ] \\
& ~~+ \cdots \\
& ~~+[V_{t+1}(p_{11}[|e|],\tau({\omega_{k+1}}),\tau({\omega_{k+2}}),\cdots,\tau({\omega_{n-1}}),\tau({\omega_{n}}),\tau({\omega_{n}}),p_{01}[k-|e|])\\
&~~~~   -V_{t+1}(p_{11}[|e|],\tau({\omega_{k+1}}),\tau({\omega_{k+2}}),\cdots,\tau({\omega_{n-1}}),\tau({\omega_{n}}),p_{01},p_{01}[k-|e|]) ] \\
&\geq 0
\end{split}
\end{equation}
where, noticing $\tau(\omega_i)$ increases with $\omega_i$ and $p_{01}\leq \tau({\omega}) \leq p_{11}$ when $p_{11}>p_{01}$, and each item in brackets is larger than or equal to zero according to the induction hypothesis.

We have $V_t(\Omega(t))$ increases monotonically with $\omega_i$ through the two cases and complete the proof.
\end{proof}

\begin{lemma}
\label{lemma:v_standard}
$V_t(\Omega(t))$ is a standard reward function.
\end{lemma}
\begin{proof}
It is obvious that $V_t(\Omega(t))$ is a standard reward function according to its definition and Lemma~\ref{lemma:symmetry},~\ref{assumption:affine} and~\ref{lemma:monotonicity}.
\end{proof}
In this section, we analyze the feature of a class of standard reward function, $V_t(\Omega(t))$, of which the optimality of greedy policy will be explored in the next section.

\section{Optimality of Greedy Policy for Standard Reward Function}
\label{section:optimality}
In this section, we first give the main theorem of optimality for the class of standard reward function, which states the sufficient condition of discounted factor for the optimality of greedy policy. After introducing some useful lemmas, we will give the complete proof of the theorem of optimality.

Let $\omega_{-i}$ denote the believe vector except the $i$th element $\omega_i$, and define
\begin{eqnarray*}
\begin{cases}
\displaystyle F'_{max}\triangleq \max_{1\leq i \leq k}\{ \frac{\partial{F(\omega_1(t),...,\omega_i(t),...,\omega_n(t))}}{\partial{\omega_i(t)}} \}=\max_{i\in{\cal N},\ \omega_{-i}\in[0,1]^{N-1}} \ \{F(1, \omega_{-i})-F(0, \omega_{-i})\}, \\
\displaystyle F'_{min}\triangleq \min_{1\leq i \leq k}\{ \frac{\partial{F(\omega_1(t),...,\omega_i(t),...,\omega_n(t))}}{\partial{\omega_i(t)}} \}=\min_{i\in{\cal N},\ \omega_{-i}\in[0,1]^{N-1}} \ \{F(1, \omega_{-i})-F(0, \omega_{-i})\}.
\end{cases}
\end{eqnarray*}
It is easy to verify that $F'_{max} \geq F'_{min} \geq 0$ based on the three basic assumptions.

The main theorem of optimality is firstly stated as follows:

\begin{theorem}
\label{theorem:optimal_condition}
    The myopic policy is optimal for ${p_{01}}\leq{\omega_{i}(1)}\leq{p_{11}}, {1}\leq{i}\leq{N}$ if $F(\Omega(t))$ is a standard reward function, and the discounted factor $\beta$ satisfies the following condition:
    \begin{equation}
    \label{eq:optimal}
    \begin{split}
      0 \leq & \beta \leq {\frac{F_{min}^{'}}{ F_{ma x}^{'}(1-(1-p_{11})^{N-k-1})}}
    \end{split}
    \end{equation}
\end{theorem}

In order to prove the Theorem~\ref{theorem:optimal_condition}, we introduce some useful lemmas firstly. Note Lemmas~\ref{lemma:indirect_exchange},~\ref{lemma:differnce_bound} and~\ref{lemma:direct_exchange} hold under condition~\eqref{eq:optimal} in the rest of the paper.
\begin{lemma}
\label{lemma:indirect_exchange}
    If $k+1\leq{i}\leq{n-1}$, $p_{11}\geq{\omega_{i}}\geq{\omega_{i+1}}\geq{p_{01}}$, and~\eqref{eq:optimal} is satisfied,
    \begin{equation}
        V_{t}(\omega_{1},...,\omega_{k},...,\omega_{i},\omega_{i+1},,...,\omega_{n})-V_{t}(\omega_{1},...,\omega_{k},...,\omega_{i+1},\omega_{i},,...,\omega_{n})\geq 0,~~~~t=1,\cdots,T.
    \end{equation}
\end{lemma}

\begin{lemma}
\label{lemma:differnce_bound}
    For $1> {\omega_1(t)}\geq{\omega_2(t)}\geq...\geq{\omega_n(t)} > 0$, if~\eqref{eq:optimal} is satisfied, we have the following inequality for all $t=1,2,...,T$:
    \begin{equation}
         V_{t}(\omega_{1},...,\omega_{k},...,\omega_{n-1},\omega_{n})-V_{t}(\omega_{n},\omega_{1},...,\omega_{k},...,\omega_{n-1})
       \leq F_{max}^{'},~~~~t=1,\cdots,T.
    \end{equation}
\end{lemma}

\begin{comment}
\begin{lemma}
\label{lemma:direct_exchange}
    If $p_{11}\geq{x}\geq{y}\geq{p_{01}}$,
    \begin{equation*}
    \begin{split}
      ~ & W_{t}(\omega_{1},...,\omega_{k-1},x,y,...,\omega_{n})-W_{t}(\omega_{1},...,\omega_{k-1},y,x,...,\omega_{n}) \\
        & \geq (x-y)(F(\omega_{1},...,\omega_{k-1},1)-F(\omega_{1},...,\omega_{k-1},0)) \\
        & - \beta(x-y)(1-\prod_{j=k+2}^{N}(1-\omega_{j}))F_{max}^{'}\\
        & \geq (x-y)F_{min}^{'} - \beta(x-y)(1-\prod_{j=k+2}^{N}(1-\omega_{j}))F_{max}^{'}\\
        & \geq 0
    \end{split}
    \end{equation*}
\end{lemma}
\end{comment}

\begin{lemma}
\label{lemma:direct_exchange}
    If $p_{11}\geq{x}\geq{y}\geq{p_{01}}$ and~\eqref{eq:optimal} is satisfied,
    \begin{equation}
    \begin{split}
      ~ & V_{t}(\omega_{1},...,\omega_{k-1},x,y,...,\omega_{n})-V_{t}(\omega_{1},...,\omega_{k-1},y,x,...,\omega_{n}) \geq 0 ,~~~~t=1,\cdots,T.
    \end{split}
    \end{equation}
\end{lemma}

\begin{remark}
We would like to point out the complicated dependence in the following proving process that Lemma \ref{lemma:indirect_exchange} depends on Lemma 2, \ref{lemma:differnce_bound} and \ref{lemma:direct_exchange}, Lemma \ref{lemma:differnce_bound} depends on Lemma \ref{lemma:differnce_bound} and \ref{lemma:direct_exchange}, Lemma \ref{lemma:direct_exchange} depends on Lemma \ref{lemma:direct_exchange} and \ref{lemma:differnce_bound}. Therefore, we give the proof of Lemma \ref{lemma:indirect_exchange}, \ref{lemma:differnce_bound} and \ref{lemma:direct_exchange} together by backward induction over time horizon.
\end{remark}

\begin{proof}
The proving process is based on backward induction in three steps as follows:
\begin{itemize}
\item step $1$: slot $T$,
\end{itemize}
These Lemmas hold trivially in slot $T$ noticing $V_{T}(\Omega(T)=F(\Omega(T)))$.

part $1$: Lemma \ref{lemma:indirect_exchange}:
\begin{equation*}
\begin{split}
  ~ & ~V_{T}(\Omega_{1},...,\omega_{k},...,\omega_{i},\omega_{i+1},,...,\omega_{n}) - V_{T}(\omega_{1},...,\omega_{k},...,\omega_{i+1},\omega_{i},,...,\omega_{n}) \\
    & = F(\omega_{1},...,\omega_{k}) - F(\omega_{1},...,\omega_{k}) = 0
\end{split}
\end{equation*}

part $2$: Lemma \ref{lemma:differnce_bound}:
\begin{equation*}
\begin{split}
  &~V_{T}(\omega_{1},...,\omega_{k},...,\omega_{n-1},\omega_{n})-V_{T}(\omega_{n},\omega_{1},...,\omega_{k},...,\omega_{n-1})\\
  & = F(\omega_{1},...,\omega_{k-1},\omega_{k}) - F(\omega_{n},\omega_{1},...,\omega_{k-1}) \\
  & = (\omega_{k}-\omega_{n})(F(\omega_{1},...,\omega_{k-1},1)-F(\omega_{1},...,\omega_{k-1},0)) \leq F_{max}^{'}
\end{split}
\end{equation*}
where, the second equality is due to Lemma~\ref{lemma:symmetry} and~\ref{lemma:affine}.

part $3$: Lemma \ref{lemma:direct_exchange}:
\begin{equation*}
\begin{split}
  ~ & V_{T}(\omega_{1},...,\omega_{k-1},x,y,...,\omega_{n})-V_{T}(\omega_{1},...,\omega_{k-1},y,x,...,\omega_{n}) \\
    & = F(\omega_{1},...,\omega_{k-1},x)-F(\omega_{1},...,\omega_{k-1},y) \\
    & = (x-y)(F(\omega_{1},...,\omega_{k-1},1)-F(\omega_{1},...,\omega_{k-1},0)) \\
    & \geq (x-y)F_{min}^{'} \geq 0
\end{split}
\end{equation*}

\begin{itemize}
\item step $2$: slot $t+1,...,T-1$:
\end{itemize}
Now suppose at $t+1,...,T-1$, Lemma \ref{lemma:indirect_exchange} (Induction Hypothesis 1, HS1), \ref{lemma:differnce_bound} (Induction Hypothesis 2, HS2), and \ref{lemma:direct_exchange} (Induction Hypothesis 3, HS3) are true, we thus prove these Lemmas also hold in slot $t$.

\begin{itemize}
\item step $3$: slot $t$:
\end{itemize}

 part $1$: Lemma \ref{lemma:indirect_exchange}:
\begin{equation*}
\begin{split}
  ~ & V_{t}(\omega_{1},...,\omega_{k},...,\omega_{i},\omega_{i+1},,...,\omega_{n}) - V_{t}(\omega_{1},...,\omega_{k},...,\omega_{i+1},\omega_{i},,...,\omega_{n}) \\
    & = (\omega_{i}-\omega_{i+1})(V_{t}(\omega_{1},...,\omega_{i-1},1,0,\omega_{i+2},,...,\omega_{n}) - V_{t}(\omega_{1},...,\omega_{i-1},0,1,\omega_{i+2},...,\omega_{n}))\\
    & = (\omega_{i}-\omega_{i+1}) \{ F(\omega_{1},...,\omega_{k}) +\beta \sum_{e\in{\mathcal{P}(a^k(t))}} \prod_{i\in{e}}\omega_i\prod_{j\in{a^k(t)\backslash{e}}}(1-\omega_j)\times\\ &~~~~~~~~V_{t+1}(p_{11}[|e|],\tau({\omega_{k+1}}),...,\tau({\omega_{i-1}}),p_{11},p_{01},\tau({\omega_{i+2}}),...,\tau({\omega_{n}}),p_{01}[k-|e|]) \} \\
    &~~ - (\omega_{i}-\omega_{i+1}) \{ F(\omega_{1},...,\omega_{k}) +\beta \sum_{e\in{\mathcal{P}(a^k(t))}} \prod_{i\in{e}}\omega_i\prod_{j\in{a^k(t)\backslash{e}}}(1-\omega_j)\times\\ &~~~~~~~~V_{t+1}(p_{11}[|e|],\tau({\omega_{k+1}}),...,\tau({\omega_{i-1}}),p_{01},p_{11},\tau({\omega_{i+2}}),...,\tau({\omega_{n}}),p_{01}[k-|e|]) \} \\
    & = (\omega_{i}-\omega_{i+1}) \beta \sum_{e\in{\mathcal{P}(a^k(t))}} \prod_{i\in{e}}\omega_i \prod_{j\in{a^k(t)\backslash{e}}}(1-\omega_j) \{ \\
    &  ~~~~V_{t+1}(p_{11}[|e|],\tau({\omega_{k+1}}),...,\tau({\omega_{i-1}}),p_{11},p_{01},\tau({\omega_{i+2}}),...,\tau({\omega_{n}}),p_{01}[k-|e|]) \\
    & ~~~~- V_{t+1}(p_{11}[|e|],\tau({\omega_{k+1}}),...,\tau({\omega_{i-1}}),p_{01},p_{11},\tau({\omega_{i+2}}),...,\tau({\omega_{n}}),p_{01}[k-|e|]) \} \\
    & \geq 0
\end{split}
\end{equation*}
where, $a^k(t)=\{\omega_1,...,\omega_k\}$, the first equality is due to Lemma~\ref{lemma:affine}, the inequality is due to the IH1 if $|e|+i-k-1 \geq k$, and IH3 if $|e|+i-k-1 = k-1$, and the Lemma \ref{lemma:symmetry} if $|e|+i-k-1 < k-1$.

part $2$: Lemma \ref{lemma:differnce_bound}:

we have the following decomposition according to the Lemma \ref{lemma:affine}
\begin{equation*}
    \begin{split}
      &V_{t}(\omega_{1},\omega_{2},...,\omega_{k-1},\omega_{k},...,\omega_{n-1},\omega_{n})-V_{t}(\omega_{n},\omega_{1},\omega_{2},...,\omega_{k-1},\omega_{k},...,\omega_{n-1})\\
      & = \omega_{k}\omega_{n}((\omega_{1},\omega_{2},...,\omega_{k-1},1,\omega_{k+1},...,\omega_{n-1},1)-V_{t}(1,\omega_{1},\omega_{2},...,\omega_{k-1},1,\omega_{k+1},...,\omega_{n-1}))\\
      & + \omega_{k}(1-\omega_{n})((\omega_{1},\omega_{2},...,\omega_{k-1},1,\omega_{k+1},...,\omega_{n-1},0)-V_{t}(0,\omega_{1},\omega_{2},...,\omega_{k-1},1,\omega_{k+1},...,\omega_{n-1}))\\
      & + (1-\omega_{k})\omega_{n}((\omega_{1},\omega_{2},...,\omega_{k-1},0,\omega_{k+1},...,\omega_{n-1},1)-V_{t}(1,\omega_{1},\omega_{2},...,\omega_{k-1},0,\omega_{k+1},...,\omega_{n-1}))\\
      & + (1-\omega_{k})(1-\omega_{n})((\omega_{1},\omega_{2},...,\omega_{k-1},0,\omega_{k+1},...,\omega_{n-1},0)-V_{t}(0,\omega_{1},\omega_{2},...,\omega_{k-1},0,\omega_{k+1},...,\omega_{n-1}))
    \end{split}
\end{equation*}

Therefore, we analyze the above formulation through four cases as follows:

Case 1. The first term of the right hand of the above formulation where channels $k$ and $n$ have the state realization "1" and "1", respectively, and denote $a^{k-1}(t)=\{\omega_{1},\omega_{2},...,\omega_{k-1}\}$, we thus have
\begin{equation*}
\begin{split}
  &V_{t}(\omega_{1},\omega_{2},...,\omega_{k-1},1,\omega_{k+1},...,\omega_{n-1},1)-V_{t}(1,\omega_{1},\omega_{2},...,\omega_{k-1},1,\omega_{k+1},...,\omega_{n-1}) \\
  & = F(\omega_{1},\omega_{2},...,\omega_{k-1},1)-F(1,\omega_{1},\omega_{2},...,\omega_{k-1})\\
  & ~~+ \beta \sum_{e\in{\mathcal{P}(a^{k-1}(t))}} \prod_{i\in{e}}\omega_i \prod_{j\in{a^{k-1}(t)\backslash{e}}}(1-\omega_j) \{ \\
  &  ~~V_{t+1}(p_{11}[|e|],p_{11},\tau({\omega_{k+1}}),...,\tau({\omega_{n-1}}),\tau({\omega_{n}}),p_{01}[k-1-|e|]) \\
  & ~~- V_{t+1}(p_{11}[|e|],p_{11},\tau({\omega_{k}}),\tau({\omega_{k+1}}),...\tau({\omega_{n-1}}),p_{01}[k-1-|e|]) \} \\
    & = \beta \sum_{e\in{\mathcal{P}(a^{k-1}(t))}} \prod_{i\in{e}}\omega_i \prod_{j\in{a^{k-1}(t)\backslash{e}}}(1-\omega_j) \{ \\
    & ~~ V_{t+1}(p_{11}[|e|],p_{11},\tau({\omega_{k+1}}),...,\tau({\omega_{n-1}}),p_{11},p_{01}[k-1-|e|]) \\
    & ~~- V_{t+1}(p_{11}[|e|],p_{11},p_{11},\tau({\omega_{k+1}}),...\tau({\omega_{n-1}}),p_{01}[k-1-|e|]) \} \\
    & \leq 0  \leq F_{max}^{'}
\end{split}
\end{equation*}
where, the first inequality is due to the Lemma~\ref{lemma:monotonicity} according to the similar way as~\eqref{eq:monotonicity_case2}.

Case 2. The second term of the right hand of the above formulation where channels $k$ and $n$ have the state realization "1" and "0", respectively, and denote $a^{k-1}(t)=\{\omega_{1},\omega_{2},...,\omega_{k-1}\}$,
\begin{equation*}
\begin{split}
  &V_{t}(\omega_{1},\omega_{2},...,\omega_{k-1},1,\omega_{k+1},...,\omega_{n-1},0)-V_{t}(0,\omega_{1},\omega_{2},...,\omega_{k-1},1,\omega_{k+1},...,\omega_{n-1}) \\
  & = F(\omega_{1},\omega_{2},...,\omega_{k-1},1)-F(0,\omega_{1},\omega_{2},...,\omega_{k-1})\\
  & ~~+ \beta \sum_{e\in{\mathcal{P}(a^{k-1}(t))}} \prod_{i\in{e}}\omega_i \prod_{j\in{a^{k-1}(t)\backslash{e}}}(1-\omega_j) \{ \\
  & ~~~~ V_{t+1}(p_{11}[|e|],p_{11},\tau({\omega_{k+1}}),...,\tau({\omega_{n-1}}),p_{01},,p_{01}[k-1-|e|]) \\
  & ~~~~~~- V_{t+1}(p_{11}[|e|],p_{11},\tau({\omega_{k+1}}),...\tau({\omega_{n-1}}),p_{01},p_{01}[k-1-|e|]) \} \\
  & = F(\omega_{1},\omega_{2},...,\omega_{k-1},1)-F(0,\omega_{1},\omega_{2},...,\omega_{k-1})\\
  & \leq F_{max}^{'}
\end{split}
\end{equation*}

Case 3. The third term of the right hand of the above formulation where channels $k$ and $n$ have the state realization "0" and "1", respectively, and denote $a^{k-1}(t)=\{\omega_{1},\omega_{2},...,\omega_{k-1}\}$,
\begin{equation*}
\begin{split}
  &V_{t}(\omega_{1},\omega_{2},...,\omega_{k-1},0,\omega_{k+1},...,\omega_{n-1},1)-V_{t}(1,\omega_{1},\omega_{2},...,\omega_{k-1},0,\omega_{k+1},...,\omega_{n-1}) \\
  & = F(\omega_{1},\omega_{2},...,\omega_{k-1},0)-F(1,\omega_{1},\omega_{2},...,\omega_{k-1})\\
  & ~~~~+ \beta \sum_{e\in{\mathcal{P}(a^{k-1}(t))}} \prod_{i\in{e}}\omega_i \prod_{j\in{a^{k-1}(t)\backslash{e}}}(1-\omega_j) \{ \\
  & ~~~~ V_{t+1}(p_{11}[|e|],\tau({\omega_{k+1}}),...,\tau({\omega_{n-1}}),p_{11},p_{01},,p_{01}[k-1-|e|]) \\
  & ~~~~~~- V_{t+1}(p_{11}[|e|],p_{11},p_{01},\tau({\omega_{k+1}}),...\tau({\omega_{n-1}}),p_{01}[k-1-|e|]) \} \\
  & \leq F(\omega_{1},\omega_{2},...,\omega_{k-1},0)-F(1,\omega_{1},\omega_{2},...,\omega_{k-1})\\
  & ~~~~+ \beta \sum_{e\in{\mathcal{P}(a^{k-1}(t))}} \prod_{i\in{e}}\omega_i \prod_{j\in{a^{k-1}(t)\backslash{e}}}(1-\omega_j) \{ \\
  & ~~~~ V_{t+1}(p_{11}[|e|],\tau({\omega_{k+1}}),...,\tau({\omega_{n-1}}),p_{11},p_{01},,p_{01}[k-1-|e|]) \\
  & ~~~~~~- V_{t+1}(p_{11}[|e|],p_{01},p_{11},\tau({\omega_{k+1}}),...\tau({\omega_{n-1}}),p_{01}[k-1-|e|]) \} \\
  & = F(\omega_{1},\omega_{2},...,\omega_{k-1},0)-F(1,\omega_{1},\omega_{2},...,\omega_{k-1})\\
  & ~~~~+ \beta \sum_{e\in{\mathcal{P}(a^{k-1}(t))}} \prod_{i\in{e}}\omega_i \prod_{j\in{a^{k-1}(t)\backslash{e}}}(1-\omega_j) \{ \\
  & ~~~~ V_{t+1}(p_{11}[|e|],\tau({\omega_{k+1}}),...,\tau({\omega_{n-1}}),p_{11},p_{01},,p_{01}[k-1-|e|]) \\
  & ~~~~~~- V_{t+1}(p_{01},p_{11}[|e|],p_{11},\tau({\omega_{k+1}}),...\tau({\omega_{n-1}}),p_{01}[k-1-|e|]) \} \\
  & \leq F(\omega_{1},\omega_{2},...,\omega_{k-1},0)-F(1,\omega_{1},\omega_{2},...,\omega_{k-1})\\
  & ~~~~+ \beta \sum_{e\in{\mathcal{P}(a^{k-1}(t))}} \prod_{i\in{e}}\omega_i ~~~~~~\prod_{j\in{a^{k-1}(t)\backslash{e}}}(1-\omega_j) \{ \\
  & ~~~~ V_{t+1}(p_{11}[|e|],\tau({\omega_{k+1}}),...,\tau({\omega_{n-1}}),p_{11},p_{01},,p_{01}[k-1-|e|]) \\
  &  ~~~~+ F_{max}^{'}- V_{t+1}(p_{11}[|e|],p_{11},\tau({\omega_{k+1}}),...\tau({\omega_{n-1}}),p_{01},p_{01}[k-1-|e|]) \} \\
  & \leq F(\omega_{1},\omega_{2},...,\omega_{k-1},0)-F(1,\omega_{1},\omega_{2},...,\omega_{k-1}) + \beta F_{max}^{'} \\
  & \leq \beta F_{max}^{'} \leq F_{max}^{'}
\end{split}
\end{equation*}
where, the first inequality is due to IH3 when $|e|+1=k$, the second one due to the IH2, and the second equality  due to Lemma \ref{lemma:symmetry} when $|e|+1<k$, noticing $0\leq |e|\leq k-1$.

Case 4. The forth term of the right hand of the above formulation where channels $k$ and $n$ have the state realization "0" and "0", respectively, and denote $a^{k-1}(t)=\{\omega_{1},\omega_{2},...,\omega_{k-1}\}$,
\begin{equation*}
\begin{split}
  &V_{t}(\omega_{1},\omega_{2},...,\omega_{k-1},0,\omega_{k+1},...,\omega_{n-1},0)-V_{t}(0,\omega_{1},\omega_{2},...,\omega_{k-1},0,\omega_{k+1},...,\omega_{n-1}) \\
  & = F(\omega_{1},\omega_{2},...,\omega_{k-1},0)-F(0,\omega_{1},\omega_{2},...,\omega_{k-1})\\
  & ~~~~+ \beta \sum_{e\in{\mathcal{P}(a^{k-1}(t))}} \prod_{i\in{e}}\omega_i \prod_{j\in{a^{k-1}(t)\backslash{e}}}(1-\omega_j) \{ \\
    & ~~~~~ V_{t+1}(p_{11}[|e|],\tau({\omega_{k+1}}),...,\tau({\omega_{n-1}}),p_{01},p_{01},p_{01}[k-1-|e|]) \\
    & ~~~~~~- V_{t+1}(p_{11}[|e|],p_{01},\tau({\omega_{k+1}}),...\tau({\omega_{n-1}}),p_{01},p_{01}[k-1-|e|]) \} \\
    & = \beta \sum_{e\in{\mathcal{P}(a^{k-1}(t))}} \prod_{i\in{e}}\omega_i \prod_{j\in{a^{k-1}(t)\backslash{e}}}(1-\omega_j) \{ \\
    & ~~~~ V_{t+1}(p_{11}[|e|],\tau({\omega_{k+1}}),...,\tau({\omega_{n-1}}),p_{01},p_{01},p_{01}[k-1-|e|]) \\
    & ~~~~~~- V_{t+1}(p_{11}[|e|],p_{01},\tau({\omega_{k+1}}),...\tau({\omega_{n-1}}),p_{01},p_{01}[k-1-|e|]) \} \\
    & = \beta \sum_{e\in{\mathcal{P}(a^{k-1}(t))}} \prod_{i\in{e}}\omega_i \prod_{j\in{a^{k-1}(t)\backslash{e}}}(1-\omega_j) \{ \\
    & ~~~~ V_{t+1}(p_{11}[|e|],\tau({\omega_{k+1}}),...,\tau({\omega_{n-2}}),\tau({\omega_{n-1}}),p_{01},p_{01},p_{01}[k-1-|e|]) \\
    & ~~~~~~- V_{t+1}(p_{01},p_{11}[|e|],\tau({\omega_{k+1}}),...\tau({\omega_{n-1}}),p_{01},p_{01}[k-1-|e|]) \} \\
    & \leq \beta \sum_{e\in{\mathcal{P}(a^{k-1}(t))}} \prod_{i\in{e}}\omega_i \prod_{j\in{a^{k-1}(t)\backslash{e}}}(1-\omega_j) \{ \\
    & ~~~~ V_{t+1}(p_{11}[|e|],\tau({\omega_{k+1}}),\tau({\omega_{k+2}}),...,\tau({\omega_{n-2}}),\tau({\omega_{n-1}}),p_{01},p_{01},p_{01}[k-1-|e|]) \\
    & ~~~~~~+ F_{max}^{'}- V_{t+1}(p_{11}[|e|],\tau({\omega_{k+1}}),...\tau({\omega_{n-1}}),p_{01},p_{01},p_{01}[k-1-|e|]) \} \\
    &\leq \beta F_{max}^{'}
\end{split}
\end{equation*}
where, the first inequality is due to the IH2 and the third equality is due to Lemma~\ref{assumption:symmetry}.

Combing the results of cases 1, 2, 3, and 4, we have
\begin{equation*}
    \begin{split}
      &V_{t}(\omega_{1},\omega_{2},...,\omega_{k-1},\omega_{k},...,\omega_{n-1},\omega_{n})-V_{t}(\omega_{n},\omega_{1},\omega_{2},...,\omega_{k-1},\omega_{k},...,\omega_{n-1})\\
      & \leq  \omega_{k}\omega_{n} 0 + \omega_{k}(1-\omega_{n})F_{max}^{'} + (1-\omega_{k})\omega_{n} \beta F_{max}^{'}   + (1-\omega_{k})(1-\omega_{n}) \beta F_{max}^{'} \\
      & \leq F_{max}^{'}
    \end{split}
\end{equation*}
To this end, we complete the proof of Lemma~\ref{lemma:differnce_bound}.

part $3$: Lemma \ref{lemma:direct_exchange}:

\begin{equation*}
\begin{split}
  ~ & V_{t}(\omega_{1},...,\omega_{k-1},x,y,...,\omega_{n})-V_{t}(\omega_{1},...,\omega_{k-1},y,x,...,\omega_{n}) \\
    & = (x-y)(V_{t}(\omega_{1},...,\omega_{k-1},1,0,...,\omega_{n})-V_{t}(\omega_{1},...,\omega_{k-1},0,1,...,\omega_{n}))\\
    & = (x-y)(F(\omega_{1},...,\omega_{k-1},1)-(\omega_{1},...,\omega_{k-1},0)) \\
  & ~~~~+ (x-y) \beta \sum_{e\in{\mathcal{P}(a^{k-1}(t))}} \prod_{i\in{e}}\omega_i \prod_{j\in{a^{k-1}(t)\backslash{e}}}(1-\omega_j) \{ \\
    &  ~~~~V_{t+1}(p_{11}[|e|],p_{11},p_{01},\tau({\omega_{k+2}}),...,\tau({\omega_{n}}),p_{01}[k-1-|e|]) \\
    & ~~~~~~- V_{t+1}(p_{11}[|e|],p_{11},\tau({\omega_{k+2}}),...\tau({\omega_{n}}),p_{01},p_{01}[k-1-|e|]) \} \\
    & \geq (x-y)(F(\omega_{1},...,\omega_{k-1},1)-F(\omega_{1},...,\omega_{k-1},0)) \\
    & - \beta(x-y)(1-\prod_{j=k+2}^{N}(1-\omega_{j}))F_{max}^{'}\\
     & \geq (x-y)F_{min}^{'} - \beta(x-y)(1-\prod_{j=k+2}^{N}(1-\omega_{j}))F_{max}^{'} \\
     & = (x-y)(1-\prod_{j=k+2}^{N}(1-\omega_{j}))F_{max}^{'} (\frac{F_{min}^{'}}F_{max}^{'}(1-\prod_{j=k+2}^{N}(1-\omega_{j})) - \beta) \\
     & \geq (x-y)(1-\prod_{j=k+2}^{N}(1-\omega_{j})) (\frac{F_{min}^{'}}{F_{max}^{'}(1-(1-p_{11})^{N-k-1})} - \beta) \\
     & \geq 0
\end{split}
\end{equation*}
where, the third inequality is due to condition~\eqref{eq:optimal} and the first inequality is due to the following inequality formulation,
\begin{equation}
\label{eq:optimal_n_1}
\begin{split}
    \Delta V & = V_{t+1}(p_{11}[|e|],p_{11},p_{01},\tau({\omega_{k+2}}),...,\tau({\omega_{n}}),p_{01}[k-1-|e|]) \\
    & ~~~~ - V_{t+1}(p_{11}[|e|],p_{11},\tau({\omega_{k+2}}),...\tau({\omega_{n}}),p_{01},p_{01}[k-1-|e|]) \} \\
    & \geq -(1-\prod_{j=k+2}^{N}(1-\omega_{j}))F_{max}^{'}
\end{split}
\end{equation}
Note, if $\tau({\omega_{k+2}})=\cdots\tau({\omega_{n}})=p_{01}$, then $\Delta V = 0$. This event happens with the probability equaling to $\prod_{j=k+2}^{N}(1-\omega_{j})$. Thus with the probability $1-\prod_{j=k+2}^{N}(1-\omega_{j})$, exists at least $i$, $k+2 \le i \le n$ such that $\tau({\omega_{i}}) > p_{01}$. According to the IH2 and IH4, we have $\Delta V \geq -F_{max}^{'}$ with probability $1-\prod_{j=k+2}^{N}(1-\omega_{j})$, which is~\eqref{eq:optimal_n_1}.

Therefore, we finish the whole proving process of Lemmas \ref{lemma:indirect_exchange}, \ref{lemma:differnce_bound}, and \ref{lemma:direct_exchange}.

%1). $0 \leq W_{t}(\omega_{1},...,\omega_{j},x,y,...,\omega_{n})-W_{t}(\omega_{1},...,\omega_{j},y,x,...,\omega_{n}) \leq \beta(x-y) {((1-p_{01})^{k-1})(p_{11}-p_{01})} \leq {((1-p_{01})^{k-1})(p_{11}-p_{01})}$
%\begin{equation*}
%    \begin{split}
%      ~ & \leq W_{t}(\omega_{1},...,\omega_{j},x,y,...,\omega_{n})-W_{t}(\omega_{1},...,\omega_{j},y,x,...,\omega_{n})\\
%        & =
%    \end{split}
%\end{equation*}
%
%
%2). On one hand, we have,
%\begin{equation*}
%\begin{split}
%   ~ & W_{t}(\omega_{1},...,\omega_{k-1},x,y,...,\omega_{n})-W_{t}(\omega_{1},...,\omega_{k-1},y,x,...,\omega_{n}) \\
%   & \geq (x-y)\prod_{i=1}^{k-1}(1-\omega_{i}) - \beta(x-y)(1-\prod_{j=k+2}^{N}(1-\omega_{j})) {((1-p_{01})^{k-1})(p_{11}-p_{01})}\\
%   & \geq (x-y)( ((1-p_{11})^{k-1})- \beta (1-(1-p_{11})^{N-k-1})((1-p_{01})^{k-1})(p_{11}-p_{01}) )\\
%   & > (x-y)( ((1-p_{11})^{k-1})- \beta (1-(1-p_{11})^{N-k-1}) ) \\
%   & \geq (x-y)(1-(1-p_{11})^{N-k-1}) ( \frac{((1-p_{11})^{k-1})}{(1-(1-p_{11})^{N-k-1})}- \beta  )\\
%   & \geq (x-y)(1-(1-p_{11})^{N-k-1}) ( \frac{((1-p_{11})^{k-1})}{(1-(1-p_{11})^{N-k-1})}- \beta  ) \\
%   & \geq 0
%\end{split}
%\end{equation*}
%On the other hand,
%\begin{equation*}
%\begin{split}
%   ~ & W_{t}(\omega_{1},...,\omega_{k-1},x,y,...,\omega_{n})-W_{t}(\omega_{1},...,\omega_{k-1},y,x,...,\omega_{n}) \\
%   & \leq (x-y)\prod_{i=1}^{k-1}(1-\omega_{i}) \\
%   & \leq {((1-p_{01})^{k-1})(p_{11}-p_{01})}
%\end{split}
%\end{equation*}

\end{proof}

After obtaining the Lemmas \ref{lemma:indirect_exchange}, \ref{lemma:differnce_bound}, and \ref{lemma:direct_exchange}, we are ready to prove the Theorem \ref{theorem:optimal_condition}.
\begin{proof}
The basic approach is by induction on $t$. It is obvious that the myopic policy is optimal at $T$. Now, assuming the optimality of the myopic policy for $t+1,...,T-1$, we shall show the myopic policy is also optimal for $t$. Denote $\{i_1,\cdots,i_n\}$ as any one of permutations of $\mathcal{N}$. To prove the optimality of greedy policy in slot $t$, we need to prove
\begin{equation}
\label{eq:bubble_sort}
    V_t(\omega_1,\cdots,\omega_k,\cdots,\omega_n) \geq V_t(\omega_{i_{1}},\cdots,\omega_{i_{k}},\cdots,\omega_{i_{n}})
\end{equation}
The proving process is same as the Bubble Sort algorithm, comparing each pair of adjacent items and swapping them if they are in the wrong order according to Lemma~\ref{lemma:symmetry},~\ref{lemma:indirect_exchange} and ~\ref{lemma:direct_exchange} until no swaps are needed, which indicates that the list is sorted to $V_t(\omega_1,\cdots,\omega_k,\cdots,\omega_n)$. The optimality of greedy policy at slot $t$ is guaranteed. Therefore, the Theorem~\ref{theorem:optimal_condition} is concluded.
\end{proof}

\begin{corollary}
\label{corollary:1_channel}
    The greedy policy is optimal if choosing $1$ out of $n$ channels for $0 < \beta \leq 1$ if $p_{11}>p_{01}$.
\end{corollary}
\begin{proof}
When $k=1$, according to Lemmas \ref{lemma:symmetry}, \ref{lemma:affine} and \ref{lemma:monotonicity}, we have $F(\Omega(t))=a\omega_i(t)$, $a>0$, thence,
\begin{equation}
   {\frac{F_{min}^{'}}{ F_{ma x}^{'}(1-(1-p_{11})^{N-k-1})}}= {\frac{1}{(1-(1-p_{11})^{N-2})}} > 1
\end{equation}

According to Theorem \ref{theorem:optimal_condition}, we have the conclusion.
\end{proof}

\begin{corollary}
\label{corollary:n_minus_1_channel}
    The greedy policy is optimal if choosing $n-1$ out of $n$ channels for $0 < \beta \leq 1$.
\end{corollary}
\begin{proof}
In case of $k = n-1$, we have
\begin{equation}
   \left [{\frac{F_{min}^{'}}{ F_{ma x}^{'}(1-(1-p_{11})^{N-k-1})}} \right ]_{k=N-1}\longrightarrow \infty
\end{equation}
Hence, the greedy policy is optimal according to Theorem \ref{theorem:optimal_condition}.
\end{proof}

\section{Applications in Cognitive Radio Network}
\label{section:application}
To illustrate the application of the mathematical results derived in the previous section, three typical scenarios~\cite{Sahmand09conf}~\cite{Wang11} described by standard reward function are presented here, which demonstrate that the different optimality conditions are completely due to different forms of the immediate reward function.
\subsection{Application 1}
An application is in a synchronously slotted cognitive radio network where a SU can opportunistically access a set $\mathcal{N}$ of $N$ i.i.d. channels partially occupied by PUs. The state of each channel $i$ in time slot $t$, denoted by $S_{i}(t)$, is modeled by a discrete time two-state Markov chain. At the beginning of each slot $t$, the SU selects a subset $\mathcal{A}(t)$ of channels to sense. If at least one of the sensed channels is in the idle state (i.e., unoccupied by any PU), the SU transmits its packet and collects one unit of reward. Otherwise, the SU cannot transmit, thus obtaining no reward. These decision procedure is repeated for each slot. The objective is to maximize the average reward over $T$ slots, that is to say, the discounted factor $\beta=1$.

\begin{comment}Let $\mathbf{{\mathcal A}}=\{{\mathcal A}(t), 1\le t\le T\}$, the optimization problem of the SU $P_{SU}$, when the SU is allowed to sense $k$ channels, is formally defined as follows\footnote{The more generic utility function can be formed by integrating the discount factor and allowing $T=+\infty$. Our results still hold in these cases.}:
\begin{equation}
P_{SU}: \ \max_{\substack{\mathbf{{\mathcal A}}\in{\mathcal N}^T \\ |{\mathcal A}(t)|=k}}\frac{1}{T}\sum_{t=1}^T {\beta}^t \left[1-\prod_{i\in{\mathcal A}(t)}(1-\omega_i(t))\right],
\end{equation}
where $\omega_i(t)$ is the conditional probability that $S_i(t)=1$ given the past actions and observations\footnote{$\omega_i(0)$ can be set to $\frac{p_{01}}{p_{01}+p_{11}}$ if no information about the initial system state is available.}.\end{comment}

Obviously, we have the immediate reward function as follows:
\begin{equation*}
   F(\Omega(t))=1-\prod_{i\in{\mathcal A}(t)}(1-\omega_{i}(t))
\end{equation*}
Therefore, the greedy policy is to choose the best $k$ channels by~\eqref{eq:greedy_policy}.
According to Theorem \ref{theorem:optimal_condition}, we have
$F_{max}^{'}=(1-p_{01})^{k-1}$, $F_{min}^{'}=(1-p_{11})^{k-1}$ if $p_{01} \leq \omega_i(0) \leq p_{11}$, $1\leq{i}\leq n$. Therefore the greedy policy, choosing the best $k$ out of $n$ channels, is optimal if the discounted factor $\beta$ satisfies the following condition:
\begin{equation*}
    \begin{split}
      0 \leq & \beta \leq \frac{(1-p_{11})^{k-1}}{ (1-p_{01})^{k-1}(1-(1-p_{11})^{N-k-1})}
    \end{split}
    \end{equation*}
Obviously, the upper bound cannot achieve 1 generally. Thus, the greedy policy, in general, is not optimal for the average reward over time horizon proved in our previous work~\cite{Wang11}.
In particular, the greedy policy, choosing the best $k=1$ or $n-1$ out of $n$ channels is optimal for $\beta = 1$ according to the corollary~\ref{corollary:1_channel} and~\ref{corollary:n_minus_1_channel}.

\subsection{Application 2}
Consider the problem of probing $n$ independent Markov chains. Each one has two states--good (1) and bad (0)--with transition probabilities $p_{11}, p_{01}$ across chain. Assuming $p_{11}>p_{01}$. A player selects $k$ chains to probe according to its preference (policy) and obtain a reward for each probed chain in the good state. We assume that the reward is affine function of the probability of the selected channel in the good state, i.e., $u_i(t)= a\omega_i(t), a>0$, then we have the immediate reward function as follows:
\begin{equation*}
    F(\Omega(t))= a \sum_{i=1}^{n}\omega_i(t)
\end{equation*}

Since $F_{max}^{'}= F_{min}^{'}=a$, thus,
\begin{equation*}
    \begin{split}
      0 \leq & \beta \leq 1 < \frac{1}{(1-(1-p_{11})^{N-k-1})}
    \end{split}
    \end{equation*}
we have the following conclusion about this problem by Theorem \ref{theorem:optimal_condition}:
\begin{lemma}
\label{lemma:mab}
    The greedy policy of choosing the first $k$ best channels is optimal for $0 < \beta \leq 1$.
\end{lemma}
Obviously, this result is consistent with~\cite{Sahmand09}~\cite{Sahmand09conf}.

\subsection{Application 3}
Consider the scenario where a player detects $n$ independent Markov chains. Each one has two states--good (1) and bad (0)--with transition probabilities $p_{11}, p_{01}$ ($p_{11}>p_{01}$) across chain. The player selects $k$ chains to detect according to its policy and obtain one unit of reward if all detected channels are good; otherwise , no reward. We assume that the probability of $i$ channel in good state at time $t$ is $\omega_i(t)$, then we have the immediate reward function as follows:
\begin{equation*}
    F(\Omega(t))=  \Pi_{i=1}^{n} \omega_i(t)
\end{equation*}
Therefore, the greedy policy is to detect the first $k$ best channels, and $F_{max}^{'}=p_{11}^{k-1} $, $F_{min}^{'}=p_{01}^{k-1}$. We have the following conclusion by Theorem \ref{theorem:optimal_condition}:
\begin{equation*}
    \begin{split}
      0 \leq & \beta \leq \frac{p_{01}^{k-1}}{ p_{11}^{k-1} (1-(1-p_{11})^{n-k-1})}
    \end{split}
    \end{equation*}
So in case of $1<k<n-1$ the greedy policy is not optimal generally for $\beta=1$, while choosing the best $k=1 $ or $k=n-1$ out of $n$ channels is optimal for $0 < \beta \leq 1$.

\section{Conclusion}
\label{section:conclusion}
In this paper, we considered a class of POMDP problem arisen in the fields of cognitive radio network, server scheduling, and downlink scheduling in cellular systems, characterized by the so-called standard reward function. For this class of POMDP, we establish the optimal condition of the greedy policy only focusing the maximization of the immediate reward. The technical approach analyzing this problem is purely mathematical, and thus is general for other models involving the recursive backward induction on the time horizon. The future direction is to investigate non i.i.d Markov chain model through the proposed method, and another more challenging work is to extend the standard reward function by dropping at least one of three basic assumptions.

\appendices
\section{Proof of Lemma ~\ref{lemma:future_reward_symmetry}}
\label{appendix:future_reward_symmetry}
\begin{lemma}
\label{lemma:future_reward_symmetry}
Assume $a^k(t)=\{\omega_1(t),\cdots,\omega_k(t) \}$, $K_{t}(\Omega(t))$
  is symmetric about $\omega_i(t),\omega_j(t)$ for all $1 \leq i,j \leq k$, that is,
\begin{equation*}
    K_{t}(\omega_{1}(t),\cdots,\omega_{i}(t),\cdots,\omega_{j}(t),\cdots,\omega_{n}(t))=K_{t}(\omega_{1}(t),\cdots,\omega_{j}(t),\cdots,\omega_{i}(t),\cdots,\omega_{n}(t))
\end{equation*}
\end{lemma}
\begin{proof}
Let
\begin{equation}
   K_{t}^m(\Omega(t))=\sum_{\substack {e\in{\mathcal{P}(a^k(t))} \\ |e|=m}} \prod_{i\in{e}}\omega_i \prod_{j\in{a^k(t)\backslash{e}}}(1-\omega_j)V_{t+1}(p_{11}[|e|],\tau({\omega_{k+1}}),\cdots,\tau({\omega_{n}}),p_{01}[k-|e|])
\end{equation}
Therefore,
\begin{equation}
    K_{t}(\Omega(t))= \sum_{m=0}^{k}K_{t}^m(\Omega(t))
\end{equation}
Since $V_{t+1}(p_{11}[|e|],\tau({\omega_{k+1}}),\cdots,\tau({\omega_{n}}),p_{01}[k])$ is unrelated with $a^k(t)$, we only need to prove the $k+1$ coefficients is symmetric about $\omega_i(t),\omega_j(t)$ for all $1 \leq i,j \leq k$, that is,
\begin{equation*}
    \mathcal{C}_t^m = \sum_{\substack {e\in{\mathcal{P}(a^k(t))} \\ |e|=m}} \prod_{i\in{e}}\omega_i \prod_{j\in{a^k(t)\backslash{e}}}(1-\omega_j), ~~~ 0 \leq m \leq k
\end{equation*}
is symmetric about $\omega_i(t),\omega_j(t)$. Based on the feature of power set $\mathcal{P}(a^k(t))$, it is simple to verify that $\mathcal{C}_t^m$, ($0 \leq m \leq k$) is symmetric about any two $\omega_i(t),\omega_j(t) \in a^k(t)$. Therefore, $K_{t}(\Omega(t))$
  is symmetric about $\omega_i(t),\omega_j(t)\in a^k(t)$.

\begin{comment}
$\mathcal{K}_{t}^m(\Omega(t))=\sum_{e\in{\mathcal{P}(a^k(t))},|e|=m} \prod_{i\in{e}}\omega_i \prod_{j\in{a^k(t)\backslash{e}}}(1-\omega_j)V_{t+1}(p_{11}[|e|],\tau({\omega_{k+1}}),...,\tau({\omega_{n}}),p_{01}[k-|e|])$, .
$|e|=0$, that is, $e = \phi$, it is obvious that $\prod_{j\in{a^k(t)}}(1-\omega_j)$ is symmetric. Since $V_{t+1}(\tau({\omega_{k+1}}),...,\tau({\omega_{n}}),p_{01}[k])$ is a constant, we have that $\prod_{j\in{a^k(t)}}(1-\omega_j)V_{t+1}(\tau({\omega_{k+1}}),...,\tau({\omega_{n}}),p_{01}[k])$ is symmetric.

$|e|=1$, $\sum_{e\in{\mathcal{P}(a^k(t))},|e|=1} \prod_{i\in{e}}\omega_i \prod_{j\in{a^k(t)\backslash{e}}}(1-\omega_j)V_{t+1}(p_{11}[|e|],\tau({\omega_{k+1}}),...,\tau({\omega_{n}}),p_{01}[k-|e|])$

$|e|=|a^k(t)|$, $e =a^k(t)$, it is obvious that $\prod_{j\in{a^k(t)}}\omega_j$ is symmetric. Since $V_{t+1}(p_{11}[k],\tau({\omega_{k+1}}),...,\tau({\omega_{n}}))$ also is a constant,  $\prod_{j\in{a^k(t)}}\omega_jV_{t+1}(p_{11}[k],\tau({\omega_{k+1}}),...,\tau({\omega_{n}}))$ is symmetric.
\end{comment}

\end{proof}

\ifCLASSOPTIONcaptionsoff
  \newpage
\fi

\bibliographystyle{unsrt}
\bibliography{reference}

\begin{thebibliography}{1}

\bibitem{Whittle80}
P.~Whittle.
\newblock Multi-armed bandits and the gittins index.
\newblock {\em Journal of Royal Statistical Society}, Series B, 42(2):143--149,
  1980.

\bibitem{Qzhao07}
Q.~Zhao, L.~Tong, A.~Swami, and Y.~Chen.
\newblock Decentralized cognitive mac for opportunistic spectrum access in ad
  hoc networks: A pomdp framework.
\newblock {\em IEEE J. Sel. Areas Commun.}, 25(3):589--600, Apr. 2007.

\bibitem{Papadimitriou99}
C.~H. Papadimitriou and J.~N. Tsitsiklis.
\newblock The complexity of optimal queueing network control.
\newblock {\em Mathematics of Operations Research}, 24(2):293--305, 1999.

\bibitem{Guha07}
S.~Guha and K.~Munagala.
\newblock Approximation algorithms for partial-information based stochastic
  control with markovian rewards.
\newblock In {\em Proc. IEEE Symposium on Foundations of Computer Science
  (FOCS)}, Providence, RI, 2007.

\bibitem{Guha09}
S.~Guha and K.~Munagala.
\newblock Approximation algorithms for restless bandit problems.
\newblock In {\em Proc. ACM-SIAM Symposium on Discrete Algorithms (SODA)}, New
  York, 2009.

\bibitem{Kliu10Index}
K.~Liu and Q.~Zhao.
\newblock Indexability of restless bandit problems and optimality of whittle
  index for dynamic multichannel access.
\newblock {\em IEEE Trans. Inf. Theory}, 56(11):5547--5567, Nov. 2000.

\bibitem{Sahmand09}
S.~Ahmand, M.~Liu, T.~Javidi, Q.~zhao, and B.~Krishnamachari.
\newblock Optimality of myopic sensing in multichannel opportunistic access.
\newblock {\em IEEE Trans. Inf. Theory}, 55(9):4040--4050, Sep. 2009.

\bibitem{Sahmand09conf}
S.~Ahmad and M.~Liu.
\newblock Multi-channel opportunistic access: A case of restless bandits with
  multiple players.
\newblock In {\em Proc. Allerton Conf. Commun. Control Comput}, pages
  1361--1368, Oct. 2009.

\bibitem{Wang11}
K.~Wang and L.~Chen.
\newblock On the optimality of myopic sensing in multi-channel opportunistic
  access: the case of sensing multiple channels.
\newblock {\em In submission to IEEE Transactions on Communication, available
  on Computing Research Repository (CoRR) arXiv:1103.1784v1}, 2011.

\end{thebibliography}

\end{document}